\documentclass[letterpaper, 10 pt, journal, twoside]{ieeetran}
\IEEEoverridecommandlockouts
\usepackage{microtype}
\usepackage{graphics} % for pdf, bitmapped graphics files
\usepackage{epsfig} % for postscript graphics files
\usepackage{times} % assumes new font selection scheme installed
\usepackage{amssymb}  % assumes amsmath package installed
\usepackage{bm}
\usepackage{xcolor}
\usepackage{cite}
\usepackage{array}
\usepackage{booktabs}
\usepackage{multirow}
\usepackage{gensymb,soul}
\usepackage{ulem} %to strike the words
\usepackage{hyperref}
\hypersetup{pdfstartview=FitH,
            colorlinks=true,
            linkcolor=red,
            anchorcolor=blue,
            citecolor=black
            }
\usepackage{amsthm}
\usepackage{float}
\usepackage{booktabs}
\usepackage{multirow}
\usepackage{makecell}
\usepackage{enumerate}
\usepackage{leftindex}
\usepackage{mathtools}
\usepackage{mathrsfs}
\usepackage[utf8]{inputenc}
\usepackage[caption=false,font=footnotesize]{subfig}
\usepackage{graphicx}
\usepackage{pifont}
\usepackage{soul}
\usepackage[linesnumbered,ruled]{algorithm2e}
\usepackage{bbding}
\newcommand{\xmark}{\textcolor{red}{\ding{55}}}
\newcommand{\cmark}{\textcolor{green}{\ding{51}}}

\newcounter{RNum}
\renewcommand{\theRNum}{\arabic{RNum}}
\theoremstyle{plain}

\newtheorem{theorem}{Theorem}

\newtheorem{lemma}{Lemma}

\newcommand{\Remark}{\noindent\textbf{Remark}~\refstepcounter{RNum}\textbf{\theRNum}: }
\newcommand{\NoOne}[1]{\textcolor{red}{#1}}
\newcommand{\NoTwo}[1]{\textcolor{green}{#1}}
\newcommand{\NoThree}[1]{\textcolor{blue}{#1}}
\makeatletter 
\pretocmd\@bibitem{\color{black}\csname keycolor#1\endcsname}{}{\fail}
\newcommand\citecolor[1]{\@namedef{keycolor#1}{\color{blue}}}
\makeatother
\newcommand{\re}{\mathbb{R}}

\soulregister\NoOne7
\soulregister\NoTwo7
\soulregister\NoThree7
\soulregister\Remark7

\begin{document}

\title{
Online Trajectory Optimization for Arbitrary-Shaped Mobile Robots via Polynomial Separating Hypersurfaces
}

%\author{Shuoye Li$^{1}$, Zhiyuan Song$^{1}$, Yulin Li$^{2}$, Zhihai Bi$^{1}$, and Jun Ma$^{1}$ 
        % <-this % stops a space
\author{Shuoye Li, Zhiyuan Song, Yulin Li, Zhihai Bi, and Jun Ma 
        % <-this % stops a space
\thanks{Shuoye Li, Zhiyuan Song, Yulin Li, Zhihai Bi, and Jun Ma are with the Robotics and Autonomous Systems Thrust, The Hong Kong University of Science and Technology (Guangzhou), Guangzhou 511453, China (e-mail: sli610@connect.hkust-gz.edu.cn; zsong142@connect.hkust-gz.edu.cn; zbi217@connect.hkust-gz.edu.cn; yline@connect.ust.hk; jun.ma@ust.hk).}
% \thanks{$^{2}$Yulin Li is with the Department of Electronic and Computer Engineering, The Hong Kong University of Science and Technology, Hong Kong SAR, China (e-mail: yline@connect.ust.hk).}
}%
%}%



\maketitle

\begin{abstract}

% In this work, we present a novel online trajectory optimization approach that generates safe motion for arbitrary-shaped mobile robots in cluttered and narrow environments. To overcome the conservatism of convex approximations, we theoretically establish that any two disjoint bounded closed sets in Euclidean space can be separated by polynomial hypersurfaces and leverage this property to formulate non-conservative collision constraints. 
% Specifically, we represent the robot's swept region and the clustered obstacle point clouds as bounded closed sets, and formulate a nonlinear programming (NLP) problem to jointly optimize the trajectory and the separating polynomials. 
% This approach enables the rigorous separation of arbitrary geometries without simplifying them into convex hulls, thereby maximizing the feasible workspace in highly constrained environments. 
% The resulting optimization problem can be solved efficiently via a general NLP solver to enable real-time performance. 
% Through extensive simulations and real-world experiments with non-convex shaped robots, the proposed method demonstrates its superiority over baseline methods in generating smooth, non-conservative, and collision-free trajectories through cluttered and narrow environments.
% ------yulin v1 -----------%
An emerging class of trajectory optimization methods enforces collision avoidance by jointly optimizing the robot's configuration and a separating hyperplane. However, as linear separators only apply to convex sets, these methods require convex approximations of both the robot and obstacles, which becomes an overly conservative assumption in cluttered and narrow environments. In this work, we unequivocally remove this limitation by introducing nonlinear separating hypersurfaces parameterized by polynomial functions. We first generalize the classical separating hyperplane theorem and prove that any two disjoint bounded closed sets in Euclidean space can be separated by a polynomial hypersurface, serving as the theoretical foundation for nonlinear separation of arbitrary geometries. Building on this result, we formulate a nonlinear programming (NLP) problem that jointly optimizes the robot's trajectory and the coefficients of the separating polynomials, enabling geometry-aware collision avoidance without conservative convex simplifications. The optimization remains efficiently solvable using standard NLP solvers. Simulation and real-world experiments with nonconvex robots demonstrate that our method achieves smooth, collision-free, and agile maneuvers in environments where convex-approximation baselines fail.
\end{abstract}

\section{Introduction} \label{sec:intro}
\IEEEPARstart{I}{n} recent years, the operational domains of mobile robots have broadened considerably, resulting in increased diversity of robot morphologies and heightened environmental complexity. This increasing variety in robot shapes, together with the inherent complexity of real-world environments, presents significant challenges for trajectory optimization tasks. As a fundamental component of trajectory optimization for mobile robots, collision avoidance is crucial for enabling safe navigation in confined spaces. A widely adopted geometric modeling assumption in robot planning is to treat the robot as a mass point and ensure safety either by inflating obstacles \cite{deits2015efficient,tordesillas2021faster,zhou2019robust} or enforcing a minimum distance constraint \cite{quan2023robust}. While this simplification improves the efficiency of formulating and solving the optimization problem, it inevitably leads to conservatism in planning tasks. Similar geometric simplifications, such as representing robots with ellipsoids \cite{li2020fast,nair2022collision} or convex polyhedrons \cite{tordesillas2021mader,wang2025fast,cinar2025polyhedral}, are also commonly used. These geometric simplifications tend to yield conservative solutions, frequently underestimating the feasible maneuverability of non-convex robots in complex environments.

\begin{figure}[t]	
	\centering
        % \vspace{3pt}
	% \includegraphics[trim=4cm 6cm 5.5cm 4cm, clip,width=1\linewidth]{images/converge_ill.png}
 	\includegraphics[width=1.0\linewidth]{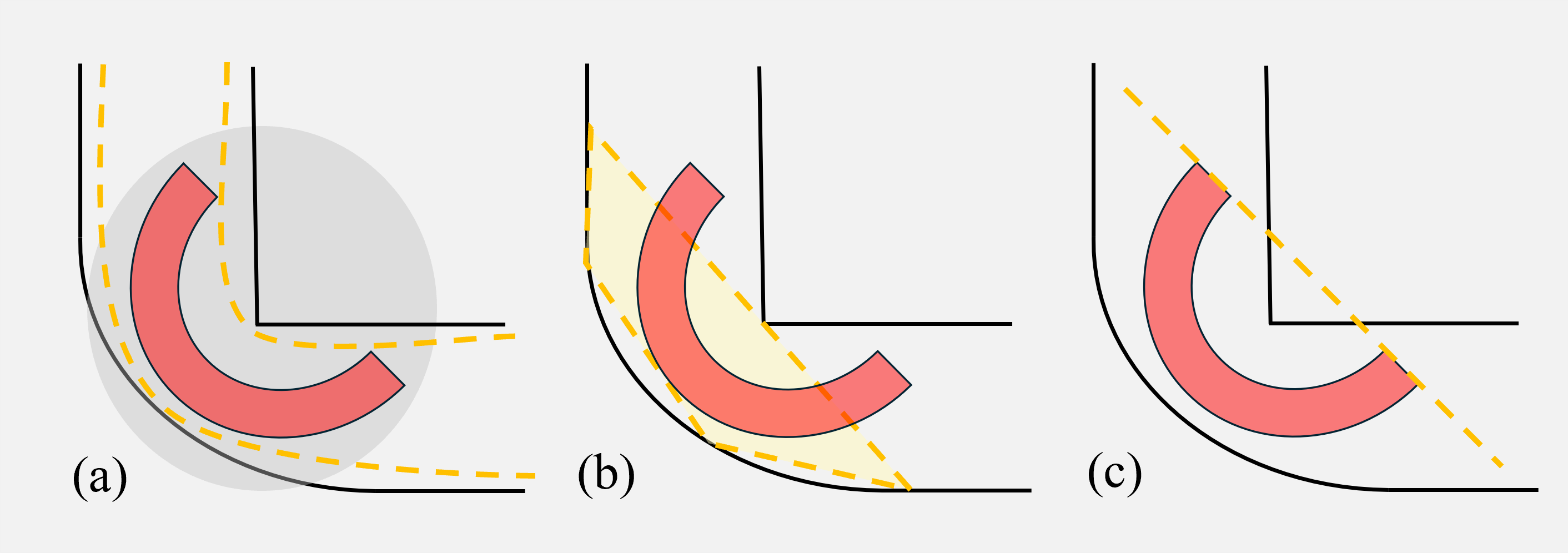}

	\caption
	{Comparison of collision avoidance constraints for a C-shaped robot in a tight cornering scenario. (a) \textbf{Proposed method}. Separating hypersurfaces (yellow curves) are generated to partition the robot and obstacles within the workspace of interest (dark gray).
    (b) \textbf{Corridor-based methods}. Relying on convex decompositions of the collision-free space, these methods fail to find a convex free region (yellow region) large enough to contain the robot fully.
    (c) \textbf{Separating hyperplane-based methods}. These methods enforce a separating hyperplane (yellow dotted line) between each obstacle and the robot's convex hull, resulting in an overly conservative solution space.
    }
	\label{fig:separating1}
\end{figure}

The separating hyperplane theorem \cite{boyd2004convex}, a basic result in convex analysis, serves as the theoretical foundation for many trajectory planning methods. They represent robots and obstacles as convex sets and employ hyperplanes to separate these entities \cite{tordesillas2021mader,nair2022collision,fan2024efficient}. By searching for such a hyperplane, these methods can efficiently ensure collision avoidance between the robot and each obstacle.
However, for robots and obstacles with inherently complex shapes, using convex hulls or other rough convex approximations can result in a significant loss of available free space. This often prevents the robot from traversing passages that would otherwise be feasible, thus limiting the planner's effectiveness in narrow environments.

To overcome the conservatism of convex approximations, we propose a novel trajectory optimization approach that extends the concept of linear separators to polynomial separating hypersurfaces, as illustrated in Fig.~\ref{fig:separating1}. We theoretically establish that any two disjoint bounded closed sets can be strictly separated by the zero-level set of a polynomial. This theoretical guarantee allows us to directly handle arbitrary robot morphologies by modeling the robot and the obstacles as such bounded closed sets. Based on this principle, we formulate a nonlinear programming (NLP) problem where the robot's profile is characterized by a set of collision points, and obstacle point clouds are grouped into clusters with extracted feature points. By enforcing separation constraints directly on these discrete representations, the resulting NLP can be efficiently solved in real time, enabling fast replanning in cluttered and narrow environments.
In summary, the main contributions of our work are as follows:
\begin{itemize}
    \item We propose a novel collision evaluation approach based on polynomial separating hypersurfaces and theoretically prove that any two disjoint bounded closed sets can be strictly separated by a polynomial hypersurface, providing a sound guarantee for separating arbitrary-shaped robots from obstacles.
    
    \item We develop a unified NLP formulation that employs polynomial separating hypersurfaces as collision constraints and jointly optimizes the robot’s trajectory and polynomial coefficients. By expressing collision avoidance as tractable, pointwise polynomial constraints over the robot's profile and clustered obstacle point clouds, our formulation significantly reduces computational complexity while accommodating arbitrary robot geometries.
    
    \item We integrate the proposed NLP formulation into an online local planner and validate its effectiveness through extensive simulations and real-world experiments. The results demonstrate that our proposed planner outperforms existing methods, enabling non-convex shaped robots to generate safe and smooth trajectories through cluttered and narrow environments in real-time.
\end{itemize}

\section{Related Works}

Typically, robot geometry in trajectory planning tasks is simplified as a mass point, a sphere, an ellipsoid, or a convex polyhedron. These abstractions are widely used to keep collision checking and trajectory optimization computationally tractable.
A variety of methods have been developed for robot collision evaluation. Among them, representative approaches include corridor-based methods \cite{ruan2022efficient,han2021fast,li2021optimization,wei2025efficient,li2024geometry,li2024collision,li2025frtree,wang2025fast}, signed distance field (SDF)-based methods \cite{geng2023robo,wang2024implicit,zhang2025universal,cinar2025polyhedral}, and separating hyperplane-based methods \cite{tordesillas2021mader,nair2022collision,fan2024efficient}, all of which aim to achieve more precise whole-body collision evaluation while maintaining computational efficiency.

Corridor-based methods construct a sequence of overlapping convex regions to approximate the robot’s free space and constrain the robot to stay within these safe regions. An early work \cite{deits2015efficient} combines IRIS-based convex segmentation with mixed-integer optimization to generate continuously collision-free polynomial trajectories. Specifically, it assigns trajectory segments within safe regions for robots modeled as spheres. A recent study \cite{li2024geometry} models robot geometry with polynomial inequalities and enforces collision-free motion via a sum-of-squares framework that constrains the robot to locally extracted convex free regions.

SDF-based methods encode obstacle information as signed distance fields and incorporate the resulting distance evaluations into the trajectory optimization problem, either as cost terms or as explicit collision-avoidance constraints. A recent work \cite{geng2023robo} proposes a Robo-centric ESDF pre-built in the robot body frame. By evaluating only relevant obstacle points, this method reduces computational time while ensuring accurate whole-body collision evaluation for arbitrary-shaped robots.
In \cite{cinar2025polyhedral}, the signed distance between polyhedrons is formulated as the optimal value of a convex program, and collision avoidance is enforced by enumerating vertex-based signed distance constraints within a mixed complementarity problem framework, enabling exact collision checking between the polyhedral robot and obstacles. However, these SDF-based methods are often computationally intensive when performing precise collision checking, making real-time replanning challenging.

Separating hyperplane-based collision checking methods determine safety by finding a hyperplane that separates the robot and each obstacle. MADER \cite{tordesillas2021mader} parameterizes robot and obstacle trajectories using the MINVO basis to obtain convex polyhedron representations in Euclidean space, and introduces separating hyperplanes as optimization variables to ensure safe motion in dynamic environments.
Although separating hyperplane-based methods are efficient and suitable for real-time replanning, their reliance on convex representations of both robot and obstacles inherently limits their ability to achieve precise planning for robots with non-convex shapes.
To address this limitation, a recent work \cite{fan2024efficient} extends the separating hyperplane concept to non-convex geometries by representing the robot and obstacles as unions of convex polyhedrons and ellipsoids. They utilize the hyperplane separation theorem to formulate differentiable separating constraints between these decomposed convex parts. By incorporating the S-procedure for boundary containment, their framework significantly reduces the number of auxiliary variables compared to standard dual reformulations, though the computational burden of NLPs remains substantial for high-frequency real-time applications.

\section{methodology}
\subsection{Polynomial Separating Hypersurface}\label{sec:theorem_part}
% In motion planning tasks, both the robot and the obstacles can be represented as bounded closed sets in Euclidean space. We establish that if the robot and an obstacle do not collide, there always exists a polynomial hypersurface that separates them, placing the two closed sets on opposite sides of the surface. The formal proof of this result is provided below, establishing the feasibility of constructing such separating hypersurfaces 、for planning purposes.
In our trajectory optimization approach, both the robot and the obstacles are modeled as bounded closed sets in Euclidean space, without assuming convexity (see Fig.~\ref{fig:separating2}). Since we aim to jointly optimize the robot’s trajectory and the polynomial-parameterized hypersurface that separates the robot from each obstacle to ensure motion safety, the feasibility of such separators becomes essential.

Classical methods rely on the separating hyperplane theorem \cite[Chapter~2.5.1, page 46]{boyd2004convex}, which guarantees a linear separator for convex sets. However, this result does not extend to nonconvex sets or nonlinear surfaces. To provide the theoretical foundation required by our approach, we show that any two disjoint closed sets can be separated by a polynomial hypersurface, with the sets lying in different connected components of its complement. This existence result not only subsumes the linear case but also ensures that our optimization problem is well-posed. A formal proof is provided below.

We first introduce a variant of Urysohn's lemma \cite[Lemma~15.6, page 102]{willard2012general}, which will be used in our analysis:

\begin{lemma}\label{lem:ury}
A space $\mathcal{S}$ is normal if and only if whenever A and B are disjoint closed sets in $\mathcal{S}$, there is a continuous function
$g:\mathcal{S}\rightarrow\re$, such that
\begin{equation*}
    g(\boldsymbol{a})=0,\ \forall\, \boldsymbol{a}\in A,
\end{equation*}
\begin{equation*}
    g(\boldsymbol{b})=1,\ \forall\, \boldsymbol{b}\in B.
\end{equation*}
\end{lemma}

While Lemma~\ref{lem:ury} guarantees separation by continuous functions,
numerical optimization requires an algebraic representation. We now extend this existence guarantee to polynomial separating hypersurfaces.

\begin{theorem}\label{separating}
For any two disjoint bounded closed sets $A$, $B \subset \mathbb{R}^n$, there exists a polynomial function $p(\boldsymbol{x})$, such that
\begin{equation*}
    p(\boldsymbol{a})<0,\ \forall\, \boldsymbol{a}\in A,
\end{equation*}
\begin{equation*}
    p(\boldsymbol{b})>0,\ \forall\, \boldsymbol{b}\in B.
\end{equation*}
\end{theorem}

\begin{proof}
Let $\mathcal{S}\in\re^n$ be a closed ball containing $A$ and $B$. Since $\mathbb{R}^n$ is a metric space (and thus a normal space) and $\mathcal{S}$ is a closed subset of $\mathbb{R}^n$, $\mathcal{S}$ is also a normal space. Because $A$ and $B$ are disjoint closed subsets of $\mathcal{S}$, according to Lemma~\ref{lem:ury}, there exists a continuous function
$g:\mathcal{S}\rightarrow\mathbb{R}$, such that
\begin{align*}
    &g(\boldsymbol{a})=0,\ \forall\, \boldsymbol{a}\in A,\\ 
    &g(\boldsymbol{b})=1,\ \forall\, \boldsymbol{b} \in B.
\end{align*}
Let $f(\boldsymbol{x}) = g(\boldsymbol{x})-\varepsilon$, $\forall\, \varepsilon\in(0,\frac{2}{3})$, then
\begin{equation} \label{eqn:proof1}
\begin{split}
    &f(\boldsymbol{a})=-\varepsilon<0,\ \forall\, \boldsymbol{a}\in A,\\
    &f(\boldsymbol{b})=1-\varepsilon>0,\ \forall\, \boldsymbol{b}\in B.
\end{split}
\end{equation}

Since $\mathcal{S}$ is a bounded closed subset in $\mathbb{R}^n$, $\mathcal{S}$ is a compact Hausdorff space. 
Let $\boldsymbol{x} = (x_1,x_2,\dots,x_n)$ be a vector in $\mathbb{R}^n$, and denote $\mathbb{R}[\boldsymbol{x}]$ as the space of all polynomial functions in $\boldsymbol{x}$ with real coefficients. We define the set of functions $\mathcal{A}$ on $\mathcal{S}$ by restricting the domain of polynomial functions in $\mathbb{R}[\boldsymbol{x}]$ to the closed ball $\mathcal{S}$:
$$\mathcal{A}:=\{p|_\mathcal{S}:p\in \mathbb{R}[\boldsymbol{x}]\}.$$
This set $\mathcal{A}$ is a subalgebra of $C(\mathcal{S}, \mathbb{R})$, the algebra of all continuous real-valued functions on $\mathcal{S}$. We verify that $\mathcal{A}$ satisfies the conditions of the Stone-Weierstrass theorem \cite{stone1948generalized}:

\begin{itemize}
    \item Separates points: For any distinct $\boldsymbol{x}, \boldsymbol{y}\in \mathcal{S}$, $\exists\, p\in \mathcal{A}$, such that $p(\boldsymbol{x})\neq p(\boldsymbol{y})$, i.e., $\mathcal{A}$ separates points on $\mathcal{S}$.
    \item Contains constants: The non-zero constant function $p(\boldsymbol{x}) = 1$ belongs to $\mathcal{A}$.
    
\end{itemize}
By Stone–Weierstrass theorem, $\mathcal{A}$ is dense in $C(\mathcal{S},\mathbb{R})$.
Thus, for any $f\in C(\mathcal{S,\mathbb{R}}),\forall\, \zeta>0,\exists\, p\in \mathcal{A}$, such that
\[
|f(\boldsymbol{x})-p(\boldsymbol{x})|<\zeta,\ \forall\, \boldsymbol{x}\in \mathcal{S}.
\]
Taking $\zeta=\varepsilon/2$, we have
\[
    |f(\boldsymbol{x}) - p(\boldsymbol{x})| < \frac{\varepsilon}{2}, \ \forall\, \boldsymbol{x} \in \mathcal{S}.
\]
It follows that
\[
    -\frac{\varepsilon}{2} < p(\boldsymbol{x}) - f(\boldsymbol{x}) < \frac{\varepsilon}{2}, \ \forall\, \boldsymbol{x} \in \mathcal{S}.
\]
Recall (\ref{eqn:proof1}) for any $\varepsilon \in (0, \frac{2}{3})$,
\begin{align*}
    &f(\boldsymbol{a})=-\varepsilon,\ \forall\, \boldsymbol{a}\in A,\\
    &f(\boldsymbol{b})=1-\varepsilon,\ \forall\, \boldsymbol{b}\in B.
\end{align*}
Therefore, for all $\boldsymbol{a} \in A$,
\[
    p(\boldsymbol{a}) < f(\boldsymbol{a}) + \frac{\varepsilon}{2} = -\varepsilon + \frac{\varepsilon}{2} = -\frac{\varepsilon}{2} < 0,
\]
and for all $\boldsymbol{b} \in B$,
\[
    p(\boldsymbol{b}) > f(\boldsymbol{b}) - \frac{\varepsilon}{2} = (1-\varepsilon) - \frac{\varepsilon}{2} = 1 - \frac{3\varepsilon}{2} > 0,
\]
where both inequalities hold for all $\varepsilon \in (0, \frac{2}{3})$.
Thus, such a polynomial function $p(\boldsymbol{x})$ exists, as required.
\end{proof}

\begin{figure}[t]	
	\centering
        % \vspace{3pt}
	% \includegraphics[trim=4cm 6cm 5.5cm 4cm, clip,width=1\linewidth]{images/converge_ill.png}
 	\includegraphics[width=0.85\linewidth]{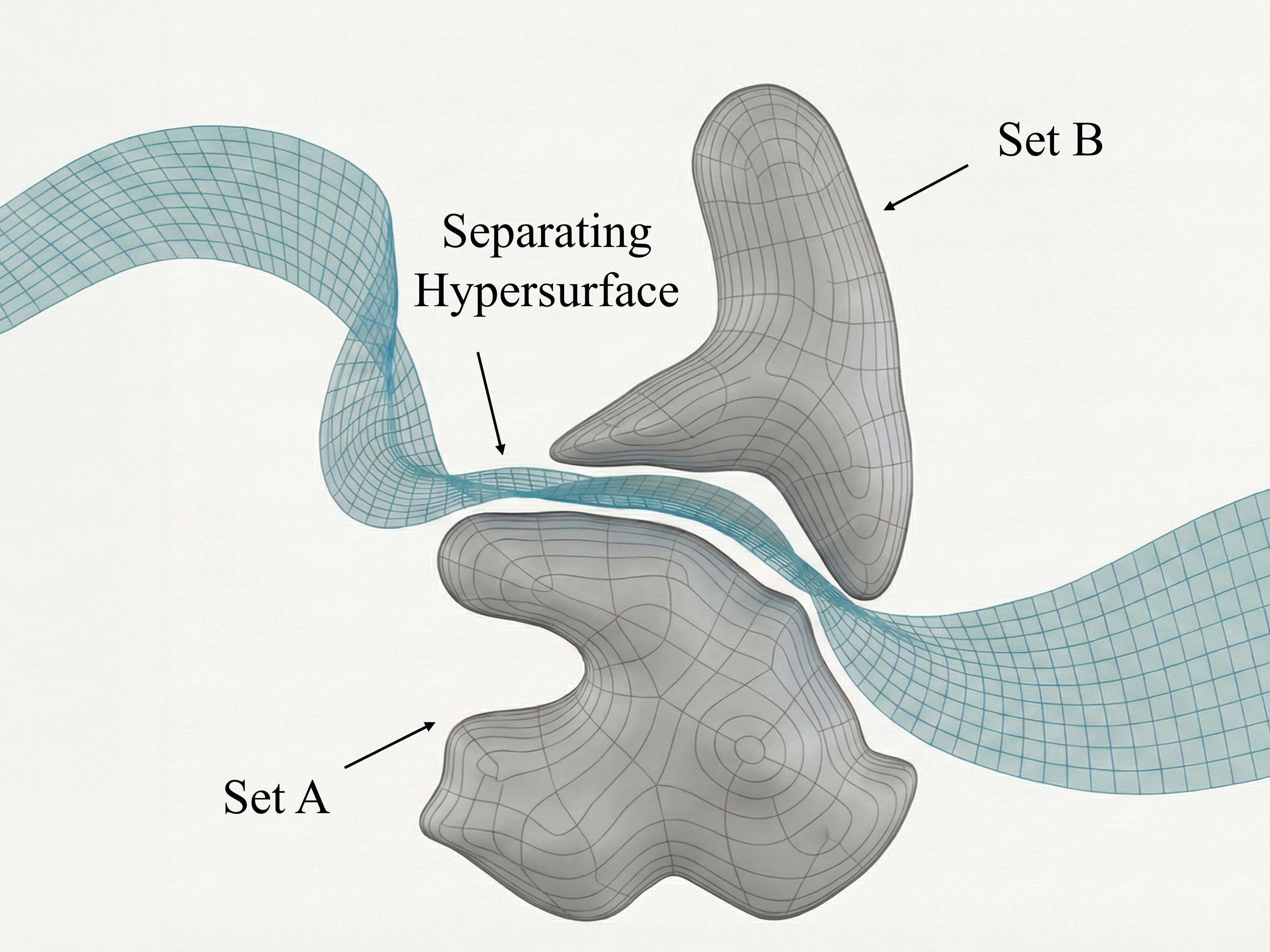}

	\caption
	{Visualization of a polynomial separating hypersurface for two bounded sets with different shapes.}
	\label{fig:separating2}
\end{figure}

By Theorem~\ref{separating}, there always exists at least one polynomial hypersurface in $\mathbb{R}^n$ capable of separating any two disjoint bounded closed sets, regardless of their specific shapes (see Fig.~\ref{fig:separating2}). In general, there may be infinitely many such hypersurfaces. In practical applications, we restrict attention to the subspace $\mathcal{A}_d \subset \mathcal{A}$, which contains all polynomials with degree up to $d$ (e.g., $d=1,2,\dots$). Note that when $d=1$, this formulation naturally reduces to the classical separating hyperplane constraint. That is, only polynomials in $\mathcal{A}_d$ are considered as candidate separating hypersurfaces. Our method is effective whenever a separating polynomial can be found within this subspace $\mathcal{A}_d$ under consideration. Empirically, polynomials of relatively low degree are typically sufficient to achieve separation between the robot and obstacles, ensuring the practical applicability of this approach in most scenarios.

\subsection{Trajectory Optimization Problem Formulation}
In this section, we apply the polynomial separating hypersurfaces introduced above to formulate the trajectory optimization problem for a holonomic ground robot. The robot's state at time $t$ is given by $\boldsymbol{q}(t)=\{x(t),y(t),\psi(t)\}$, where $(x, y)$ denotes the position and $\psi$ denotes the orientation of the robot. The velocity control input is denoted by $\boldsymbol{u}(t) = \{v_x(t), v_y(t), \omega(t)\}$, where $v_x$ and $v_y$ are the translational velocities in the robot body frame and $\omega$ is the angular velocity.

For trajectory optimization, we discretize the time horizon $T$ into $N$ steps with stamps $\tau = 0, 1, \ldots, N$. The robot's trajectory is then represented by the sequence of states $\mathcal{X} = \{\boldsymbol{q}_0, \boldsymbol{q}_1, \ldots, \boldsymbol{q}_N\}$ and control inputs $\boldsymbol{U} = \{\boldsymbol{u}_0, \boldsymbol{u}_1, \ldots, \boldsymbol{u}_{N-1}\}$, where $\boldsymbol{q}_\tau$ and $\boldsymbol{u}_\tau$ denote the state and control input at discrete time step $\tau$, respectively.

We specialize the separation method to the plane $\re^2$ and employ quadratic polynomials ($d=2$) to define the separating hypersurfaces. A quadratic polynomial $p(\boldsymbol{x})$ defined in $\re^2$ is parameterized as:
\begin{equation}\label{eq:1}
    p(\boldsymbol{x})= {\text{coef}}(p)^\top[\boldsymbol{x}]_2,
\end{equation}
where $\boldsymbol{x}=(x_1,x_2)\in\re^2$; $[\boldsymbol{x}]_2=(1,x_1,x_2,x_1^2,x_1 x_2,x_2^2)$ is the vector of all monomials in $\boldsymbol{x}$ whose degrees are not greater than 2; $\text{coef}(p)\in \re^6$ is the coefficients vector of $p$ corresponding to each monomial in $[\boldsymbol{x}]_2$.

\Remark{The proposed separating hypersurfaces is applicable not only in 2D but also in higher-dimensional spaces. Since our problem is two-dimensional, we present the definitions and formulations in 2D for clarity and computational efficiency here.}

\begin{table}[t]
    \centering
    \caption{NOMENCLATURE}
    \begin{tabular}{ c p{0.29\textwidth} }
    \toprule[1pt]
    \textbf{Symbols} & \textbf{Descriptions}\\
    \midrule
    $\boldsymbol{q}= (x,y,\psi)\in \re^3$ & The position and orientation of the robot body\\
    $\boldsymbol{u}=(v_x,v_y,\omega)\in \re^3$ & The velocity input in the robot body frame\\
    $\mathcal{S}\subset \re^2$ & The closed ball workspace of interest\\ 
    $\mathcal{O}_k\subset \re^2$ & Space occupied by the $k$th obstacle in $\mathcal{S}$ \\
    $\mathcal{B}_\tau \subset \re^2$ & Space occupied by the robot at time $\tau$ \\
    $\mathcal{V}_\mathcal{B}\subset \re^2$ & The union of the robot's occupied regions over the optimization time horizon\\
    $\boldsymbol{q}_{\text{start}}$ & Start position and orientation\\
    $\boldsymbol{q}_{\text{goal}}$ & Goal position and orientation\\
    ${p}_k\in\mathcal{A}_2$ & The polynomial used to separate the robot and the obstacle $\mathcal{O}_k$\\
    $\text{coef}(p)$ & The coefficients vector of polynomial $p$\\
    \bottomrule[1pt]
    \end{tabular}
    \label{tab:concept_definition}
\end{table}

We list relevant notations in Table~\ref{tab:concept_definition} for convenience. With this formulation, we cast trajectory planning for the holonomic ground robot as an NLP problem, which explicitly incorporates the robot's full pose and enforces collision-free motion throughout the trajectory:
\begin{align}\label{eqn:orin}
\displaystyle \operatorname*{min}_{\boldsymbol{q}_\tau,\boldsymbol{u}_\tau,{\text{coef}}(p_k)} & \quad J= J_q +J_s+J_u+J_p \notag \\
\mathrm{s.t.}\ \ \ \ \  & \quad p_k(x_i)\geq0,~\forall x_i\in\mathcal{V}_\mathcal{B},~i=1,2,\ldots,N_{\mathcal{V}_\mathcal{B}}\notag 
\\ & \quad p_k(s_j)\leq0,~\forall s_j\in \mathcal{O}_k,~j=1,2,\ldots,N_{\mathcal{O}_k}\notag
\\&\quad \forall \mathcal{O}_k\subset \mathcal{S}, ~k=1,2,\ldots,K \notag
\\&\quad \tau=0,1,\ldots,N 
\\& \quad \boldsymbol{q}_{\tau+1}=f(\boldsymbol{q}_\tau,\boldsymbol{u}_\tau),\notag
\\ & \quad \boldsymbol{u}_{min} \leq \boldsymbol{u}_\tau\leq \boldsymbol{u}_{max},\notag
\\& \quad \tau = 0,1,\ldots,N-1 \notag
\\ & \quad \boldsymbol{q}_0=\boldsymbol{q}_{\text{start}},\notag
\end{align}
where $J_q$ denotes the trajectory tracking and terminal penalty, $J_s$ is the smoothness penalty, $J_u$ is the control input penalty, and $J_p$ is the penalty on the polynomial coefficients. $\mathcal{B}_\tau$ denotes the region occupied by the robot body at time step $\tau$, and $\mathcal{V}_\mathcal{B}$ is the union of the robot's occupied regions over the optimization time horizon. Throughout the planning horizon $T$, we consider all detectable obstacles $\mathcal{O}_k$ within the reachable workspace $\mathcal{S}$.

The essence of our trajectory optimization problem lies in augmenting the standard kinematic constraints of the ground robot with collision avoidance constraints formulated via polynomial separating hypersurfaces. In this framework, both the robot’s state variables and control inputs, as well as the coefficients of the separating polynomials, are treated as optimization variables, resulting in a unified NLP problem.

\subsubsection{Tracking and Terminal Penalty}
Given a reference trajectory $\mathcal{X}_{r}=\{\boldsymbol{q}_{0,r},\boldsymbol{q}_{1,r},\ldots,\boldsymbol{q}_{N,r}\}$ and a specified goal $\boldsymbol{q}_{\text{goal}}=(x_{\text{goal}},y_{\text{goal}},\psi_{\text{goal}})$, we define a quadratic cost to penalize the tracking error and the terminal error as follows:
\begin{equation}
    J_q = \sum_{\tau=0}^{N}\boldsymbol{e}_\tau^\top Q_{eq}\boldsymbol{e}_\tau + \boldsymbol{e}_{\text{goal}}^\top Q_{fq}\boldsymbol{e}_{\text{goal}},
\end{equation}
where the pose tracking error $\boldsymbol{e}_\tau$ at each time step $\tau=0,1,\dots,N$ is given by 
\begin{align}
\boldsymbol{e}_\tau &=(x_\tau - x_{\tau, r},~y_\tau - y_{\tau, r},~\delta \psi_\tau),\\
\delta\psi_\tau&=\mathrm{atan2}\big(\sin(\psi_\tau-\psi_{\tau,r}),~\cos(\psi_\tau-\psi_{\tau,r})\big).
\end{align}
The terminal error $\boldsymbol{e}_{\text{goal}}$ is defined as:
\begin{align}
\boldsymbol{e}_{\text{goal}} &=(x_N - x_{\text{goal}},~y_N - y_{\text{goal}},~\delta \psi_{\text{goal}}),\\
\delta\psi_{\text{goal}}&=\mathrm{atan2}\big(\sin(\psi_N-\psi_{\text{goal}}),~\cos(\psi_N-\psi_{\text{goal}})\big).
\end{align}

\subsubsection{Smoothness Penalty}

To encourage smooth motion, we introduce a smoothness penalty on accelerations and angular accelerations. Specifically, we define the smoothness cost as
\begin{equation}
    J_{s} = \sum_{\tau=0}^{N-2} 
    \left(
        \boldsymbol{a}_\tau^\top Q_{a} \boldsymbol{a}_\tau 
        + \alpha_\tau^\top Q_{\alpha} \alpha_\tau
    \right),
\end{equation}
where $\boldsymbol{a}_\tau$ denotes the linear acceleration and $\alpha_\tau$ denotes the angular acceleration at time step $\tau$. These are computed as finite differences of the velocity and angular velocity:
\begin{align}
    \boldsymbol{a}_\tau &= \boldsymbol{v}_{\tau+1} - \boldsymbol{v}_\tau, \\
    \alpha_\tau &= \omega_{\tau+1} - \omega_{\tau},
\end{align}
for $\tau = 1, 2, \ldots, N-2$, where $\boldsymbol{v}_\tau$ and $\omega_\tau$ represent the linear and angular velocities at time step $\tau$, respectively.

\subsubsection{Input Penalty}

To discourage excessive or abrupt control actions, we include an input penalty term in the cost function. The input penalty is defined as
\begin{equation}
    J_u = \sum_{\tau=0}^{N-1} \boldsymbol{u}_\tau^\top R \boldsymbol{u}_\tau,
\end{equation}
where $\boldsymbol{u}_\tau=(v_{x,\tau},v_{y,\tau},\psi_\tau)$ denotes the control input at time step $\tau$.

\subsubsection{Polynomial Coefficients Penalty}

In our trajectory optimization problem~(\ref{eqn:orin}), directly searching for polynomial coefficients within the feasible region often leads to abnormal coefficient values. For instance, excessive numerical values in the polynomial coefficients or excessive numerical differences between different terms can result in undesirable geometric properties, which in turn cause poor performance of the polynomial separating hypersurfaces. By applying a small penalty to each polynomial coefficient, we encourage smoother hypersurfaces.

Formally, the polynomial coefficients penalty is defined as
\begin{equation}
    J_p = \sum_{k=1}^{K} \text{coef}(p_k)^\top Q_p \text{coef}(p_k),
\end{equation}
where $\text{coef}(p_k)$ is the coefficients vector for the $k$-th separating polynomial, and $Q_p$ is a diagonal weighting matrix.

It is worth noting that the weights in $Q_p$ are set to be several orders of magnitude smaller than those in the tracking, terminal, smoothness or input cost terms. This ensures that the coefficient penalty serves only as a gentle regularization, preventing extreme parameter values without unduly constraining the expressiveness of the separating surface. 

Furthermore, the weights corresponding to higher-order terms are chosen to be even smaller. This design encourages geometric simplicity by suppressing unnecessary higher-order terms. When the robot's non-convexity is not strictly required for collision avoidance, the separating hypersurface naturally degenerates into a simple separating hyperplane. This adaptive behavior ensures that the planner utilizes complex nonlinear separators only when the environment's complexity demands it, thereby maintaining computational efficiency and trajectory smoothness.

\subsubsection{Robot Representation and Separating Constraints}
The organization of separating hypersurface constraints in our optimization framework is grounded in Theorem~\ref{separating}, which guarantees the existence of polynomial hypersurfaces capable of separating any two disjoint bounded closed sets in workspace $\mathcal{S}$.

At each time step $\tau$ ($\tau=0,1,\ldots,N$), the space occupied by the robot $\mathcal{B}_\tau$ and any surrounding obstacle $\mathcal{O}_k$ ($k=1,2,\ldots,K$) within $\mathcal{S}$ can be regarded as bounded closed sets. We use a polynomial $p_{k,\tau}$, as described in (\ref{eq:1}), to separate the robot $\mathcal{B}_\tau$ and the obstacle $\mathcal{O}_k$. The corresponding polynomial separation constraint can be formulated as
\begin{align}
    &p_{k,\tau}(x_{i,\tau})\geq0,~\forall x_{i,\tau}\in \mathcal{B}_\tau,\label{eq:body1}\\
    &p_{k,\tau}(s_{j})\leq0,~\forall s_{j}\in\mathcal{O}_k,~j=1,2,\ldots,N_{\mathcal{O}_k}.\label{eq:obstacle1}
\end{align}

For the NLP problem~(\ref{eqn:orin}), (\ref{eq:body1}) introduces $K\cdot N_\mathcal{B}(N\!+\!1)$ nonlinear constraints, where $K$ is the number of obstacles and $N_\mathcal{B}$ is the number of check points on the robot. (\ref{eq:obstacle1}) introduces $N_\mathcal{O}(N\!+\!1)$ linear constraints, where $N_{\mathcal{O}} = \sum_{k=1}^{K} N_{\mathcal{O}_k}$ denotes the total number of sampled points on all surrounding obstacles. In addition, the number of optimization variables amounts to $6N\! +\! 6K(N\!+\!1)\!+\!3$, making real-time trajectory optimization computationally infeasible.

To overcome this limitation, we replace the robot body in the constraints with $\mathcal{V}_\mathcal{B}$, which is the union of the robot's occupied sets over the optimization horizon. For a robot represented as a bounded closed set in Euclidean space, the union of the robot's occupied sets over the planning horizon, denoted as $\mathcal{V}_\mathcal{B}=\bigcup_{\tau=0}^N\mathcal{B}_\tau$, is also a bounded closed set in the Euclidean space.

Accordingly, we reformulate constraints (\ref{eq:body1}) and (\ref{eq:obstacle1}) into a new form as in (\ref{eqn:orin}):
\begin{align}
    &p_{k}(x_{i})\geq0,~\forall x_{i}\in \mathcal{V}_{\mathcal{B}},~i=1,2,\dots,N_{\mathcal{V}_\mathcal{B}},\label{eq:body2}\\
    &p_{k}(s_{j})\leq0,~\forall s_{j}\in\mathcal{O}_k,~j=1,2,\ldots,N_{\mathcal{O}_k}.\label{eq:obstacle2}
\end{align}

This approach reduces the number of optimization variables by $6KN$ and the number of nonlinear constraints by $KN_\mathcal{B} N$, thereby improving the trajectory safety and greatly accelerating the optimization process.

\subsection{Obstacle Representation}
In our trajectory optimization problem, obstacles within the workspace $\mathcal{S}$ are modeled as bounded closed sets, consistent with the assumptions of Theorem~\ref{separating}. To implement the collision constraints formulated in (\ref{eq:obstacle2}), it is essential to provide an efficient and practical representation of obstacles.

We represent obstacles using processed point clouds, which are well aligned with the requirements of our pointwise polynomial separation constraints. In real-time trajectory optimization, operating directly on lidar point clouds substantially reduces computational overhead relative to mesh map construction.

Specifically, the point cloud data within the workspace $\mathcal{S}$ is first downsampled using a voxel grid filter to reduce data redundancy and then segmented based on Euclidean distance clustering. This process groups the raw points into $K$ distinct obstacle clusters $\{\widehat{\mathcal{O}}_k\}_{k=1}^{K}$, where each cluster ideally corresponds to an individual obstacle or a connected obstacle component. The clustering distance can be set according to the robot's traversability, allowing the segmentation to reflect practical navigational constraints.

After clustering, we select a subset of representative points that lie close to the geometric boundary of the cluster, as these boundary points capture the shape of the obstacle more effectively.

Let $\widehat{\mathcal{O}}_k = \{\boldsymbol{o}_1, \boldsymbol{o}_2, \ldots, \boldsymbol{o}_m\}$ denote the $k$-th obstacle cluster, where each $\boldsymbol{o}_i \in \mathbb{R}^3$ is a point in Cartesian coordinates and $M_k$ is the number of all points in $\widehat{\mathcal{O}}_k$. For each point $\boldsymbol{o}_i \in \widehat{\mathcal{O}}_k$, we define a dispersion score
\begin{equation}
    c_i = \frac{1}{M_k-1} \sum_{j=1 ,j \neq i}^{M_k} \bigl\| \boldsymbol{o}_j - \boldsymbol{o}_i \bigr\|_2,
\end{equation}
which measures the average Euclidean distance from $\boldsymbol{o}_i$ to all other points in the same cluster. Intuitively, points located near the geometric exterior of the cluster tend to have larger dispersion scores, whereas interior points tend to have smaller scores.

Within each cluster $\widehat{\mathcal{O}}_k$, points are sorted in descending order of $c_i$. Starting from the largest $c_i$, we iteratively add the point $\boldsymbol{o}_i$ to the feature set $\mathcal{O}_k$ as a new feature point $s_j$ only if it is sufficiently far from all previously selected feature points, i.e.,
\begin{equation}
    \bigl\|\boldsymbol{o}_i - \boldsymbol{s}_\ell\bigr\|_2 \geq r_d,
    \quad \forall\, \ell\leq i,
\end{equation}
where $r_d > 0$ is a predefined distance threshold, and $N_{\mathcal{O}_k}$ is the resulting number of selected feature points in $\widehat{\mathcal{O}}_k$. This selection yields a compact set of geometrically peripheral and mutually well-separated feature points, providing an efficient and expressive representation of the obstacles.

\section{Results}
In this section, we validate the effectiveness of our proposed method for different trajectory optimization tasks. In simulations, we employ an L-shaped holonomic robot as the test platform, as shown in Fig.~\ref{fig:robot}(a), and evaluate its trajectories when traversing narrow passages of different widths and cluttered forest environments with varying obstacle densities. The key performance metrics are recorded and analyzed in comparison with other baseline methods. 
In the real-world experiments, we mount an L-shaped frame on a Unitree Go2 quadruped robot. The frame extends beyond the robot’s original body length and width, emulating the non-convex shapes that mobile robots may exhibit in practical applications (e.g., with manipulators or external payloads). With this non-convex profile, as shown in Fig.~\ref{fig:robot}(b), the robot employs our method to perform real-time trajectory optimization and react online to changes in the cluttered indoor environment.

All simulations are executed on a desktop computer equipped with an Intel Core i9-12900KF CPU, and in the real-world experiments, we stream data wirelessly between this desktop and the quadruped robot in real time, and perform online trajectory optimization on the desktop. 
The NLPs in (\ref{eqn:orin}) are formulated using CasADi's C++ interface \cite{andersson2019casadi} and solved by IPOPT \cite{wachter2006implementation}, with MA57 \cite{duff2004ma57} selected as the sparse linear solver within IPOPT.

\begin{figure}[ht]	
	\centering
        % \vspace{3pt}
	% \includegraphics[trim=4cm 6cm 5.5cm 4cm, clip,width=1\linewidth]{images/converge_ill.png}
 	\includegraphics[width=0.95\linewidth]{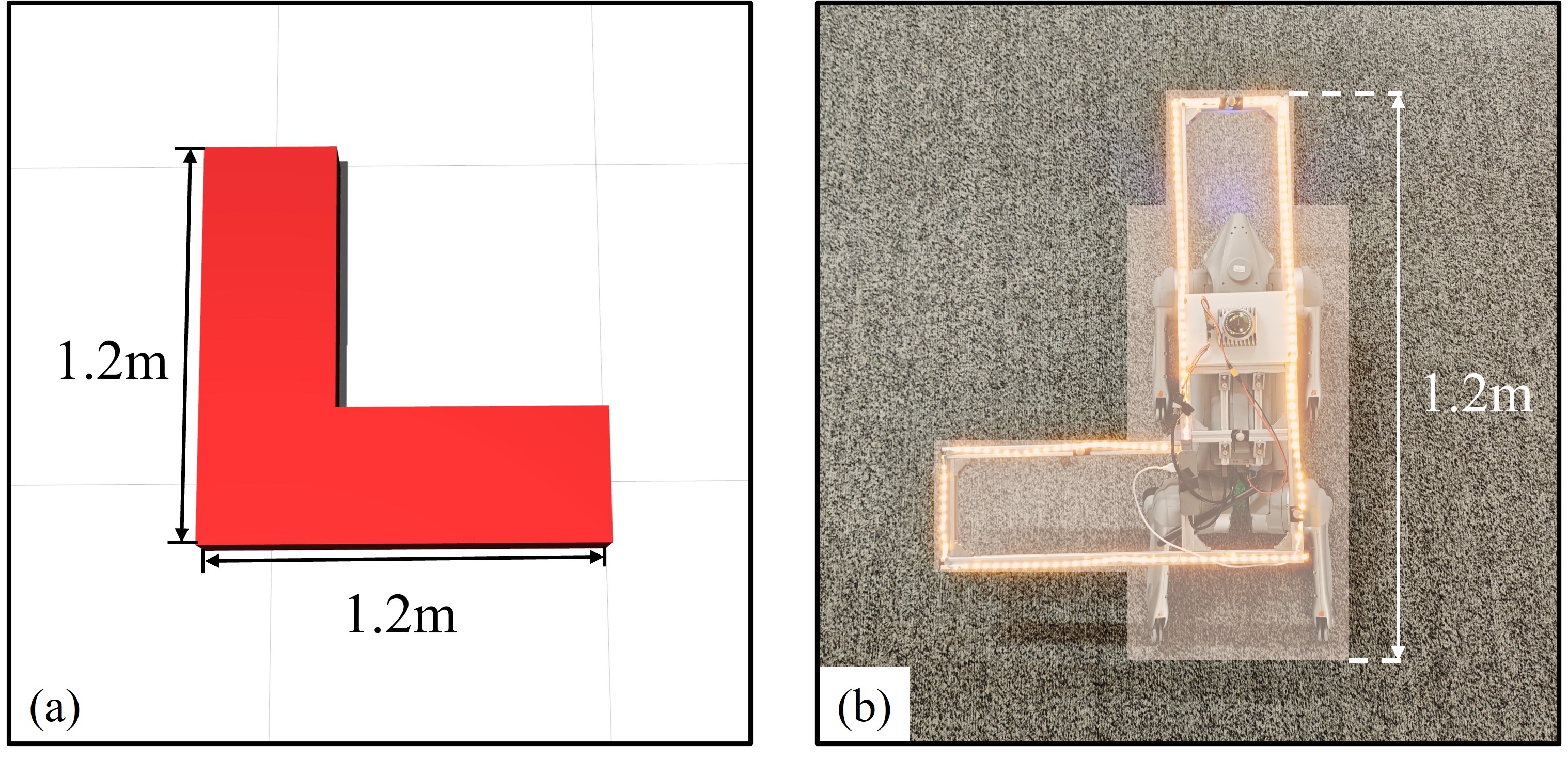}

	\caption{
    Robot models used in simulations and experiments.
    (a) L-shaped robot model adopted in simulations. 
    (b) Quadruped robot used in the real-world experiments. Its geometric profile (white region) is a non-convex shape composed of an L-shaped frame and a central rectangular body.}
	\label{fig:robot}
\end{figure}

\subsection{Simulations}
\subsubsection{Narrow Passage}
In the narrow passage simulation environment shown in Fig.~\ref{fig:passage}, the L-shaped robot must adjust its pose to traverse passages whose widths are comparable to, or even smaller than, its own length and width, which makes this scenario representative for evaluating the traversability of non-convex robots in tight spaces.

\begin{figure}[t]
  \centering
  \includegraphics[width=0.95\linewidth]{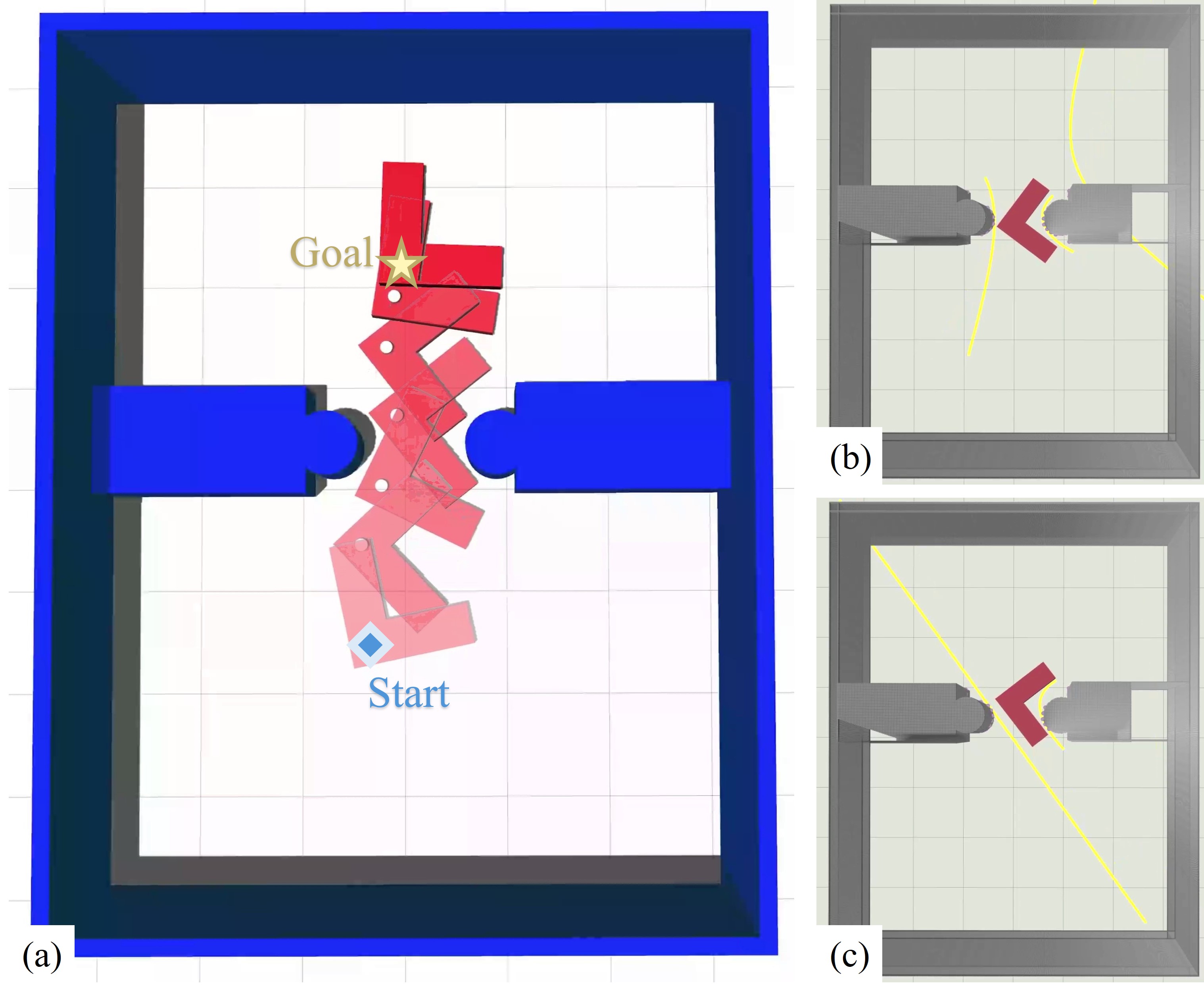}%
  \caption{
    The performance of our proposed method in the narrow passage scenario.
    (a) Visualization of the L-shaped robot navigating through a passage narrower than its own dimensions. 
    (b)-(c) Illustration of how separating hypersurfaces (yellow curves) work when traversing the narrow passage.
  }
  \label{fig:passage}
\end{figure}

\begin{table}[t]
  \centering
  \caption{Comparison of key metrics for different trajectory planners in environments with varying gap widths.}
  \setlength{\tabcolsep}{2pt}
  \begin{tabular}{cccccc}
    \toprule
   Gap Width & Metric & Ours & FASTER & DDR-OPT & Hyperplane \\
    \midrule
    \multirow{3}{*}{$w_1=1.4$\,m}
      & Path Length (m)    & \textbf{4.36} & N/A & 4.60 & 4.51 \\
      & Total Time (s)     & \textbf{9.65} & N/A & 12.31 & 9.95 \\
      & Success Rate (\%)  & 100 & 0 & 100 & 100 \\
    \midrule
    \multirow{3}{*}{$w_2=1.2$\,m}
      & Path Length (m)    & \textbf{4.36} & N/A & N/A & N/A \\
      & Total Time (s)     & \textbf{10.48} & N/A & N/A & N/A \\
      & Success Rate (\%)  & \textbf{100} & 0 & 0 & 0 \\
    \midrule
    \multirow{3}{*}{$w_3=1.0$\,m}
      & Path Length (m)    & \textbf{4.38} & N/A & N/A & N/A \\
      & Total Time (s)     & \textbf{10.05} & N/A & N/A & N/A \\
      & Success Rate (\%)  & \textbf{100} & 0 & 0 & 0 \\
    \bottomrule
  \end{tabular}
  \label{tab:passage}
\end{table}

We evaluate the proposed trajectory planner on this task and compare it against three online planners: FASTER (corridor-based) \cite{tordesillas2021faster}, DDR-OPT (ESDF-based) \cite{zhang2025universal}, and a separating hyperplane-based method modified from MADER \cite{tordesillas2021mader}. 
In these planners, FASTER models the robot as a single sphere with a diameter of 1.7\,m, while DDR-OPT tightly envelops the L-shaped body with five spheres, each with a diameter of 0.29\,m. The separating hyperplane-based method uses a convex polyhedron that tightly bounds the L-shaped robot.
All planners are run in a receding-horizon fashion to generate real-time trajectories, under the same limits on velocity and acceleration. The start and goal configurations are kept identical for all planners and all gap widths, with a straight-line distance of $d_{\text{task}} = 4.0\,\mathrm{m}$ between them.

For each method and each gap width, we conduct 10 simulation runs. The comparison focuses on path length, total mission completion time, and success rate across different gap widths (defined as the minimum distance between the obstacles forming the passage), as indicated in Table~\ref{tab:passage}, where a run is counted as successful only if the robot reaches the goal without any collision. Baselines fail primarily due to geometric conservatism or optimization limitations in highly constrained scenarios. In this scenario, our method demonstrates strong performance without relying on any front-end planner to provide reference paths. For all simulations, we simply use the straight line from the start to the goal as reference, which is not collision-free, yet the optimizer is still able to deform it into collision-free trajectories that successfully navigate through the passages. The proposed trajectory planner maintains a replanning rate of around 13\,Hz throughout the simulations.

\subsubsection{Forest}
In this subsection, we further compare the performance of our proposed method against the three planners mentioned above in forest scenario. As illustrated in Fig.~\ref{fig:forest}, forest environment consists of closely arranged cylinder obstacles. We adopt this structured layout so that the difficulty of the environment can be controlled directly through the spacing $d$ between neighboring obstacles.

\begin{figure}[t]
  \centering
  % 对本图中的 subfigure 取消 (a)(b) 标签
  \captionsetup[subfigure]{labelformat=empty}
  % 上：Gazebo overview（没有 (a)）
  \subfloat[]{%
    \includegraphics[width=0.85\linewidth]{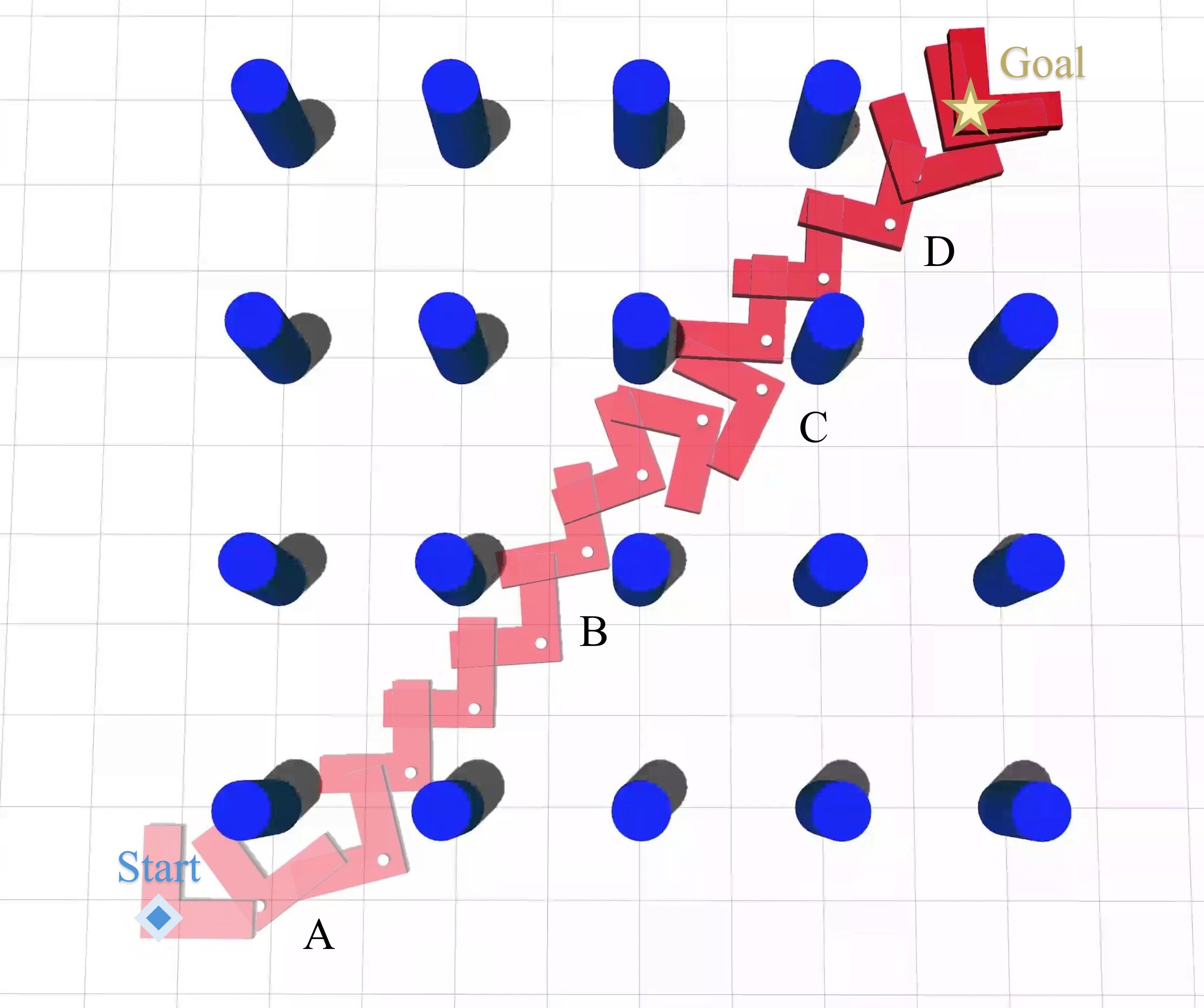}%
    \label{fig:forest1}
  }\\[-5pt]
  % 下：整块大图
  \includegraphics[width=0.9\linewidth]{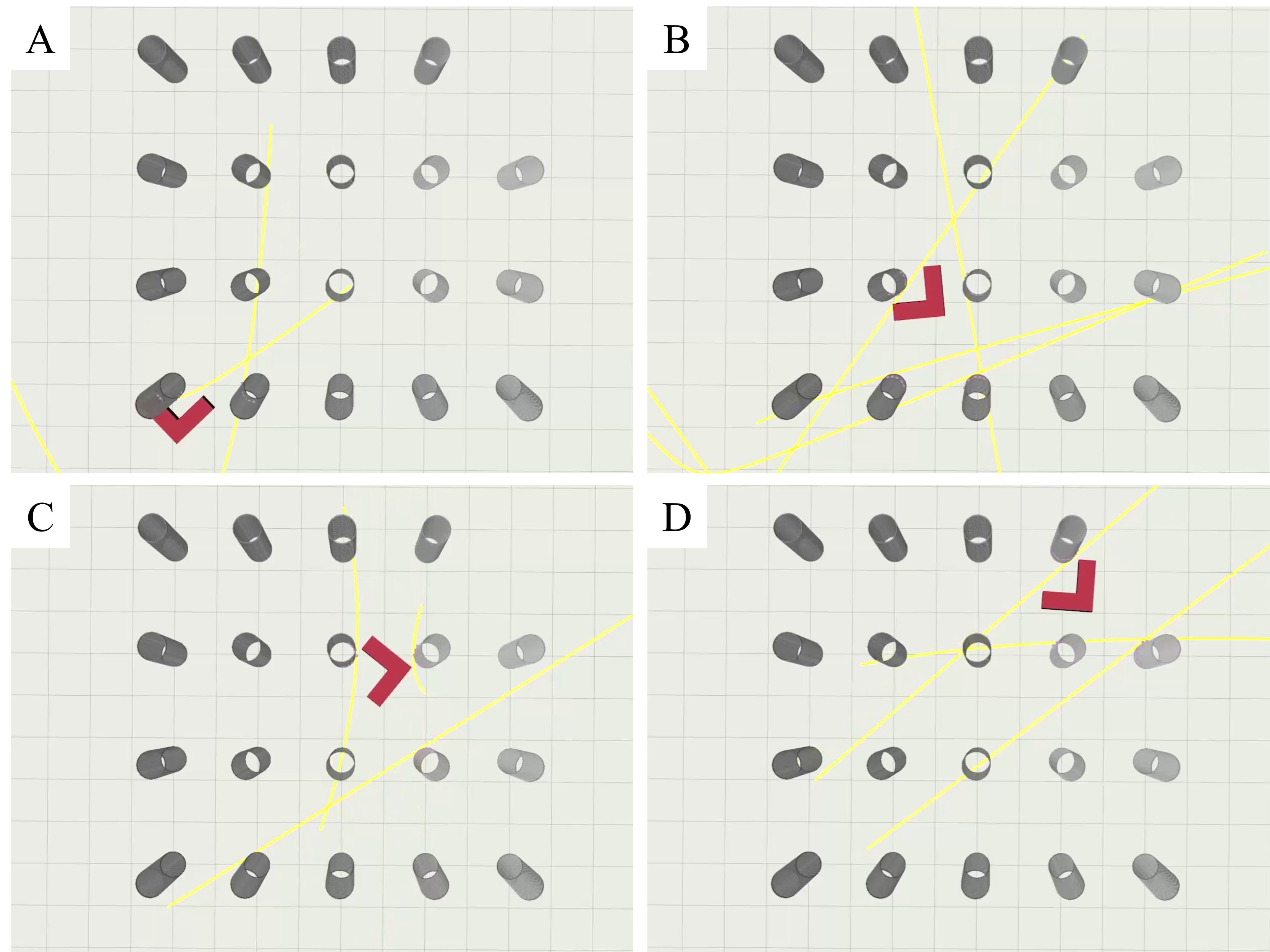}

  \caption{
    Illustration of the L-shaped robot navigating through the dense forest scenario via our proposed planner.
    The RViz snapshots at four representative time instants (A, B, C, D) during execution illustrates how the separating hypersurfaces (yellow curves) function during trajectory optimization. When the robot’s non-convexity is not necessary for collision avoidance, the hypersurfaces naturally degenerates into simple separating hyperplanes.
  }
  \label{fig:forest}
\end{figure}

Similar to the narrow-passage scenario, in the relatively relaxed forest settings with obstacle spacing $d_1 = 4.0\,\mathrm{m}$ and $d_2 = 1.6\,\mathrm{m}$, our method relies only on a straight line from the start to the goal as the reference path and still achieves strong performance. For the most constrained case $d_3 = 1.4\,\mathrm{m}$, we additionally provide a coarse reference path generated by an RRT* planner in OMPL \cite{sucan2012open}, and our optimizer refines this path into a smooth, collision-free trajectory.

In the forest experiments, we distinguish between task completion and safety. A run is considered complete if the robot reaches the goal within a reasonable time, regardless of whether collisions occur. Safety is then evaluated separately by reporting whether the corresponding trajectories are collision-free. To account for different task scales, we additionally report the ratio between the executed path length and the straight-line distance from the start to the goal, which provides a normalized measure of path efficiency.
\begin{figure}[t]	
	\centering
        % \vspace{3pt}
	% \includegraphics[trim=4cm 6cm 5.5cm 4cm, clip,width=1\linewidth]{images/converge_ill.png}
 	\includegraphics[width=1.0\linewidth]{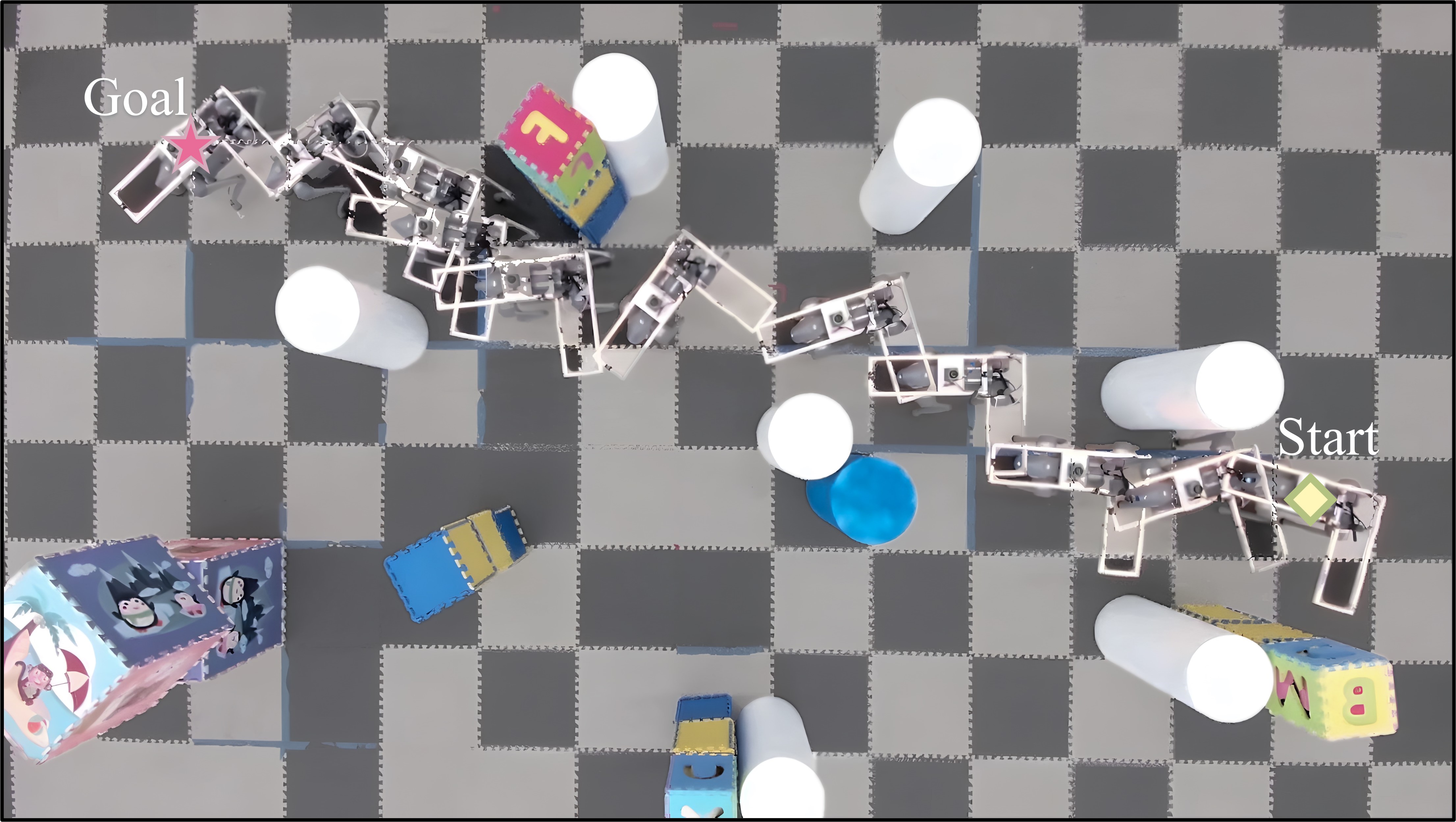}

	\caption
	{Illustration of the trajectory of a robot with a complex non-convex shape navigating through a cluttered and narrow environment.}
	\label{fig:real1}
\end{figure}

\begin{table}[t]
  \centering
  \caption{Comparison of key metrics for different trajectory planners in environments with varying forest densities.}
  \setlength{\tabcolsep}{2pt}
  \begin{tabular}{cccccc}
    \toprule
   Obstacle Spacing & Metric & Ours & FASTER & DDR-OPT & Hyperplane \\
    \midrule
    \multirow{4}{*}{$d_1=4.0$\,m}
      & Path Ratio    & \textbf{1.07} & 1.30 & \textbf{1.07} & 1.12 \\
      & Total Time (s)     & \textbf{60.39} & 72.15 & 64.63 & 63.91 \\
      & Compl. (\%)  & 100 & 100 & 100 & 100 \\
      & Collision-Free     & \cmark  & \xmark   & \cmark   & \cmark \\
    \midrule
    \multirow{4}{*}{$d_2=1.6$\,m}
      & Path Ratio    & \textbf{1.10} & N/A & 1.22 & N/A \\
      & Total Time (s)     & \textbf{41.88} & N/A & 48.36 & N/A \\
      & Compl. (\%)  & \textbf{100} & 0 & 80 & 0 \\
      & Collision-Free     & \cmark & N/A & \xmark & N/A \\
    \midrule
    \multirow{4}{*}{$d_3=1.4$\,m}
      & Path Ratio    & \textbf{1.15} & N/A & N/A & N/A \\
      & Total Time (s)     & \textbf{36.23} & N/A & N/A & N/A \\
      & Compl. (\%)  & \textbf{90} & 0 & 0 & 0 \\
      & Collision-Free     & \cmark & N/A & N/A & N/A\\
    \bottomrule
  \end{tabular}
  \label{tab:forest}
\end{table}

For each method and each obstacle spacing, 10 simulation runs are conducted. The comparison focuses on path efficiency, total completion time, task completion rate, and safety, as summarized in Table~\ref{tab:forest}. The proposed method maintains strong obstacle-avoidance and traversal performance in dense forest environments. In the forest scenario, baseline methods mainly fail due to conservative approximations and the difficulty of solving online optimization problems in complex environments. In the simulations, our trajectory optimizer operates at a frequency around 10\,Hz.

\subsection{Real-World Experiments}
To verify the practical application of the proposed method, we conduct real-world experiments with the quadruped robot in Fig.~\ref{fig:robot}(b). We model the robot as a non-convex shape consisting of a rectangle coupled with an L-shape frame to account for its actual footprint.

In terms of experimental setup, we build a cluttered and narrow indoor environment, in which the quadruped robot must continuously adjust its posture to pass through narrow passages between obstacles. During the experiments, the robot has no prior knowledge of the environment and can only rely on the onboard MID-360 LiDAR to perceive the surroundings, which provides real-time point clouds within a range of 3\,m. A straight line from the start to the goal is used as the only reference trajectory.

As shown in Fig.~\ref{fig:real1} under the Bird's Eye View (BEV), the robot successfully reaches the goal by following the collision-free trajectory generated by the online optimizer, which runs at an execution frequency of about 12\,Hz throughout the motion.

\begin{figure}[t]	
	\centering
        % \vspace{3pt}
	% \includegraphics[trim=4cm 6cm 5.5cm 4cm, clip,width=1\linewidth]{images/converge_ill.png}
 	\includegraphics[width=1.0\linewidth]{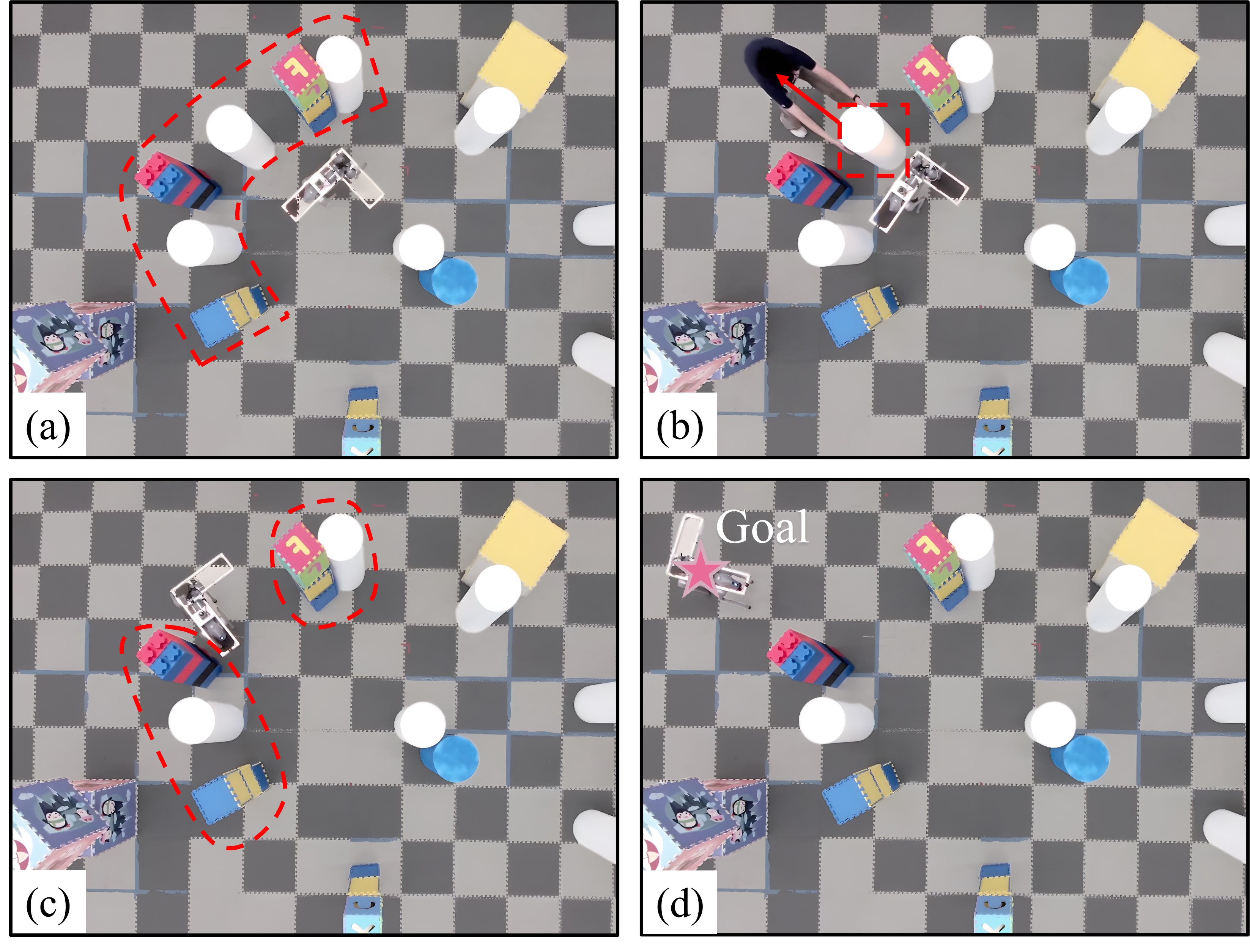}

	\caption
	{Demonstration of real-time replanning in a changing cluttered environment. 
    (a) The quadruped robot with a non-convex footprint encounters a dead end and stops. 
    (b) An obstacle is manually removed to create a narrow passage. 
    (c) The proposed planner performs real-time trajectory replanning and guides the robot through the newly opened passage. 
    (d) The robot successfully traverses the passage and reaches the goal.}
	\label{fig:real2}
\end{figure}

We further perform additional online replanning experiments to showcase the real-time capability of the proposed algorithm. As illustrated in Fig.~\ref{fig:real2}(a)-\ref{fig:real2}(d), the robot initially follows the given straight-line reference trajectory until it encounters a dead end and comes to a stop. At this point, one obstacle is manually removed, creating a narrow passage within the robot’s observable range. Using the updated map, the robot detects this new passage and replans its trajectory online, eventually reaching the goal.
The proposed method demonstrates a reliable real-time trajectory optimization capability for robots with non-convex shapes operating in cluttered and narrow environments, and exhibits excellent robustness in real-world experiments.

\section{Conclusion}
In this letter, we propose a separating hypersurface-based collision evaluation method designed for the trajectory optimization of mobile robots with arbitrary shapes. This method serves as a generalization of the classical separating hyperplane approach, extending its applicability to non-convex geometries. We then develop a comprehensive optimization formulation and integrate it into an online planner, where the obstacles are modeled from real-time point cloud measurements via clustering and specialized sampling. The proposed planner has been extensively validated through both simulations and real-world experiments. Results demonstrate its capability to operate in real-time with non-convex robot geometries and to generate efficient, collision-free trajectories in cluttered and narrow environments. The planner exhibits strong robustness in unknown environments, relying solely on local perception without prior maps. Furthermore, even when initialized with naive or infeasible straight-line references, the planner consistently yields safe and smooth motions. These results highlight the effectiveness and generality of our approach for safe and reliable motion planning in cluttered and narrow environments.

\bibliographystyle{IEEEtran}
\normalem
\bibliography{ref}

@book{boyd2004convex,
  title={Convex optimization},
  author={Boyd, Stephen P and Vandenberghe, Lieven},
  year={2004},
  publisher={Cambridge University Press}
}

@book{willard2012general,
  title={General Topology},
  author={Willard, Stephen},
  year={2004},
  publisher={Dover Publications}
}

@inproceedings{deits2015efficient,
  title={Efficient mixed-integer planning for {UAV}s in cluttered environments},
  author={Deits, Robin and Tedrake, Russ},
  booktitle={Proc. IEEE Int. Conf. Robot. Automat.},
  pages={42--49},
  year={2015},
  organization={}
}

@inproceedings{geng2023robo,
  title={Robo-centric {ESDF}: A fast and accurate whole-body collision evaluation tool for any-shape robotic planning},
  author={Geng, Shuang and Wang, Qianhao and Xie, Lei and Xu, Chao and Cao, Yanjun and Gao, Fei},
  booktitle={Proc. IEEE/RSJ Int. Conf. Intell. Robots Syst.},
  pages={290--297},
  year={2023},
  organization={IEEE}
}

@inproceedings{nair2022collision,
  title={Collision avoidance for dynamic obstacles with uncertain predictions using model predictive control},
  author={Nair, Siddharth H and Tseng, Eric H and Borrelli, Francesco},
  booktitle={Proc. IEEE Conf. Decis. Control},
  pages={5267--5272},
  year={2022},
  organization={IEEE}
}

@inproceedings{li2020fast,
  title={Fast and safe path-following control using a state-dependent directional metric},
  author={Li, Zhichao and Arslan, {\"O}m{\"u}r and Atanasov, Nikolay},
  booktitle={Proc. IEEE Int. Conf. Robot. Automat.},
  pages={6176--6182},
  year={2020},
  organization={}
}

@article{andersson2019casadi,
  title={Cas{AD}i: a software framework for nonlinear optimization and optimal control},
  author={Andersson, Joel AE and Gillis, Joris and Horn, Greg and Rawlings, James B and Diehl, Moritz},
  journal={Mathematical Programming Computation},
  volume={11},
  number={1},
  pages={1--36},
  month={Mar.},
  year={2019},
  publisher={Springer}
}

@article{sucan2012open,
  title={The open motion planning library},
  author={Sucan, Ioan A and Moll, Mark and Kavraki, Lydia E},
  journal={Robot. Autom. Mag.},
  volume={19},
  number={4},
  pages={72--82},
  month={Dec.},
  year={2012},
  publisher={IEEE}
}

@article{duff2004ma57,
  title={{MA}57---a code for the solution of sparse symmetric definite and indefinite systems},
  author={Duff, Iain S},
  journal={ACM Transactions on Mathematical Software},
  volume={30},
  number={2},
  pages={118--144},
  month={Jun.},
  year={2004},
  publisher={ACM New York, NY, USA}
}

@article{wachter2006implementation,
  title={On the implementation of an interior-point filter line-search algorithm for large-scale nonlinear programming},
  author={W{\"a}chter, Andreas and Biegler, Lorenz T},
  journal={Mathematical programming},
  volume={106},
  number={1},
  pages={25--57},
  month={Mar.},
  year={2006},
  publisher={Springer}
}

@article{stone1948generalized,
  title={The generalized {W}eierstrass approximation theorem},
  author={Stone, Marshall H},
  journal={Mathematics Magazine},
  volume={21},
  number={4},
  pages={167--184},
  month={Mar.--Apr.},
  year={1948},
  publisher={JSTOR}
}

@article{li2024geometry,
  title={Geometry-aware safety-critical local reactive controller for robot navigation in unknown and cluttered environments},
  author={Li, Yulin and Tang, Xindong and Chen, Kai and Zheng, Chunxin and Liu, Haichao and Ma, Jun},
  journal={IEEE Robot. Automat. Lett.},
  volume={9},
  number={4},
  pages={3419--3426},
  year={Apr. 2024},
  publisher={IEEE}
}

@article{li2024collision,
  title={Collision-free trajectory optimization in cluttered environments using sums-of-squares programming},
  author={Li, Yulin and Zheng, Chunxin and Chen, Kai and Xie, Yusen and Tang, Xindong and Wang, Michael Yu and Ma, Jun},
  journal={IEEE Robot. Automat. Lett.},
  year={Dec. 2024},
  volume={9},
  number={12},
  pages={11026-11033},
  publisher={IEEE}
}

@article{quan2023robust,
  title={Robust and efficient trajectory planning for formation flight in dense environments},
  author={Quan, Lun and Yin, Longji and Zhang, Tingrui and Wang, Mingyang and Wang, Ruilin and Zhong, Sheng and Zhou, Xin and Cao, Yanjun and Xu, Chao and Gao, Fei},
  journal={IEEE Trans. Robot.},
  volume={39},
  number={6},
  pages={4785--4804},
  year={Dec. 2023},
  publisher={IEEE}
}

@article{zhou2019robust,
  title={Robust and efficient quadrotor trajectory generation for fast autonomous flight},
  author={Zhou, Boyu and Gao, Fei and Wang, Luqi and Liu, Chuhao and Shen, Shaojie},
  journal={IEEE Robot. Automat. Lett.},
  volume={4},
  number={4},
  pages={3529--3536},
  year={Oct. 2019},
  publisher={IEEE}
}

@article{wei2025efficient,
  title={Efficient and Safe Trajectory Planning for Autonomous Agricultural Vehicle Headland Turning in Cluttered Orchard Environments},
  author={Wei, Peng and Peng, Chen and Lu, Wenwu and Zhu, Yuankai and Vougioukas, Stavros and Fei, Zhenghao and Ge, Zhikang},
  journal={IEEE Robot. Automat. Lett.},
  year={Mar. 2025},
  volume={10},
  number={3},
  pages={2574-2581},
  publisher={IEEE}
}

@article{ruan2022efficient,
  title={Efficient path planning in narrow passages for robots with ellipsoidal components},
  author={Ruan, Sipu and Poblete, Karen L and Wu, Hongtao and Ma, Qianli and Chirikjian, Gregory S},
  journal={IEEE Trans. Robot.},
  volume={39},
  number={1},
  pages={110--127},
  year={Feb. 2022},
  publisher={IEEE}
}

@article{han2021fast,
  title={Fast-Racing: An Open-Source Strong Baseline for {SE}(3) Planning in Autonomous Drone Racing},
  author={Han, Zhichao and Wang, Zhepei and Pan, Neng and Lin, Yi and Xu, Chao and Gao, Fei},
  journal={IEEE Robot. Automat. Lett.},
  volume={6},
  number={4},
  pages={8631--8638},
  year={Oct. 2021},
  publisher={IEEE}
}

@article{wang2024implicit,
  title={Implicit swept volume {SDF}: Enabling continuous collision-free trajectory generation for arbitrary shapes},
  author={Wang, Jingping and Zhang, Tingrui and Zhang, Qixuan and Zeng, Chuxiao and Yu, Jingyi and Xu, Chao and Xu, Lan and Gao, Fei},
  journal={ACM Transactions on Graphics},
  volume={43},
  number={4},
  pages={1--14},
  month={Jun.},
  year={2024},
  publisher={ACM New York, NY, USA}
}

@article{li2021optimization,
  title={Optimization-based trajectory planning for autonomous parking with irregularly placed obstacles: A lightweight iterative framework},
  author={Li, Bai and Acarman, Tankut and Zhang, Youmin and Ouyang, Yakun and Yaman, Cagdas and Kong, Qi and Zhong, Xiang and Peng, Xiaoyan},
  journal={IEEE Trans. Intell. Transp. Syst.},
  volume={23},
  number={8},
  pages={11970--11981},
  year={Aug. 2021},
  publisher={IEEE}
}

@article{cinar2025polyhedral,
  title={Polyhedral Collision Detection via Vertex Enumeration},
  author={Cinar, Andrew and Zhao, Yue and Laine, Forrest},
  journal={arXiv preprint arXiv:2501.13201},
  year={2025}
}

@article{zhang2025universal,
  title={Universal trajectory optimization framework for differential drive robot class},
  author={Zhang, Mengke and Chen, Nanhe and Wang, Hu and Qiu, Jianxiong and Han, Zhichao and Ren, Qiuyu and Xu, Chao and Gao, Fei and Cao, Yanjun},
  journal={IEEE Trans. Autom. Sci. Eng.},
  year={2025},
  volume={22},
  number={},
  pages={13030-13045},
  publisher={IEEE}
}

@article{wang2025fast,
  title={Fast iterative region inflation for computing large {2-D/3-D} convex regions of obstacle-free space},
  author={Wang, Qianhao and Wang, Zhepei and Wang, Mingyang and Ji, Jialin and Han, Zhichao and Wu, Tianyue and Jin, Rui and Gao, Yuman and Xu, Chao and Gao, Fei},
  journal={IEEE Trans. Robot.},
  volume={41},
  number={},
  pages={3223-3243},
  year={2025},
  publisher={IEEE}
}

@article{tordesillas2021faster,
  title={Faster: Fast and safe trajectory planner for navigation in unknown environments},
  author={Tordesillas, Jesus and Lopez, Brett T and Everett, Michael and How, Jonathan P},
  journal={IEEE Trans. Robot.},
  volume={38},
  number={2},
  pages={922--938},
  year={Apr. 2022},
  publisher={IEEE}
}

@article{li2025frtree,
  title={{FRT}ree planner: Robot navigation in cluttered and unknown environments with tree of free regions},
  author={Li, Yulin and Song, Zhicheng and Zheng, Chunxin and Bi, Zhihai and Chen, Kai and Wang, Michael Yu and Ma, Jun},
  journal={IEEE Robot. Automat. Lett.},
  year={Apr. 2025},
  volume={10},
  number={4},
  pages={3811-3818},
  publisher={IEEE}
}

@article{tordesillas2021mader,
  title={{MADER}: Trajectory planner in multiagent and dynamic environments},
  author={Tordesillas, Jesus and How, Jonathan P},
  journal={IEEE Trans. Robot.},
  volume={38},
  number={1},
  pages={463--476},
  year={Feb. 2022},
  publisher={IEEE}
}

@article{fan2024efficient,
  title={Efficient optimization-based trajectory planning for unmanned systems in confined environments},
  author={Fan, Jiayu and Murgovski, Nikolce and Liang, Jun},
  journal={IEEE Trans. Intell. Transp. Syst.},
  year={Nov. 2024},
  volume={25},
  number={11},
  pages={18547-18560},
  publisher={IEEE}
}

\end{document}